\documentclass{article} 
\usepackage{iclr2027_conference,times}
\usepackage[hidelinks]{hyperref}
\usepackage{url}
\usepackage{booktabs}
\usepackage{amsmath,amssymb,amsthm}
\usepackage{graphicx}
\newtheorem{theorem}{Theorem}

\newtheorem{remark}{Remark}
\newcommand{\Dhat}{\widehat{\Delta}}
\newcommand{\LCB}{\underline{\Delta}}

\iclrfinalcopy

\title{RE-0: Verified Recursive Improvement of \mbox{Embodied} Code-as-Policy Agents\\through Local On-Policy Distillation}

\author{Jiawei Zhang$^{1}$, Xiangrong Zhang$^{1}$, Rui Song$^{2}$,\\ \bf Huanbin Zhou$^{1}$, Chengye Song$^{2}$, Hongzhou Wang$^{1\dagger}$\\ \rm$^{1}$Jilin University, Changchun, China \quad $^{2}$Dalian University of Technology, Dalian, China
}

\begin{document}

\maketitle
{\let\thefootnote\relax\footnotetext{\normalsize$^{\dagger}$Corresponding author: \texttt{wanghongzhou@jlu.edu.cn}}}

\begin{abstract}
Code-as-Policy agents accomplish long-horizon embodied tasks by generating and executing code, yet continually improving them with teachers that are stronger but not globally reliable remains a key challenge. Existing distillation methods typically treat the teacher's complete behavior as the supervision target and thus misassign training credit on states where the teacher fails. We propose \textbf{RE-0}, a recursively verified policy improvement framework: rather than assuming that the teacher globally outperforms the student, RE-0 requests local corrections from the teacher on the student's own failure histories and checks in the environment whether each correction is genuinely beneficial; verified corrections yield immediate improvement. Building on this, we propose \textbf{RE-OPD}, which turns verified interventions into supervision for on-policy distillation. Only counterfactually verified teacher interventions provide distribution-level supervision, weighted by their measured local benefit, and the improvement they induce is projected back into the standalone student---so both where supervision is applied and how much credit the teacher receives co-evolve with the student policy. We further prove that the student's per-round gain is lower-bounded by its verified intervention gain up to verification and projection error terms. Experiments across multiple Code-as-Policy embodied tasks show that RE-0 improves both teacher-assisted execution and the standalone student, and generalizes to novel robots and scenes.
\end{abstract}

\section{Introduction}
\label{sec:intro}
Code-as-Policy agents turn language models into policies that interact with an environment through executable programs---code as the interface for compositional control, tool calling, and long-horizon reasoning. Policy errors are then strongly history-dependent: one harmful segment not only fails the current action but alters the environment state, the program namespace, and the subsequent observable history. Improving such an agent therefore poses a question more basic than \emph{what code should be generated}: \emph{where, along the execution history actually induced by the student, should corrective credit be assigned}.

Existing approaches improve code policies through expert-trajectory supervision, outcome-level reinforcement learning, or teacher distillation, but none answers whether one concrete teacher intervention on a real student failure produced a positive causal contribution. Even a strong teacher accumulates errors through perception ambiguity or earlier executions, while a teacher unable to finish the task alone may still decide better at a local failure boundary the student has reached; neither global success rate nor prefix-level reliability determines whether a teacher action deserves training credit.

We therefore reformulate teacher supervision as \emph{local counterfactual credit on student-induced histories}: on a history $h$ genuinely visited by the current student, does a bounded teacher correction beat the student's action, and does the advantage persist after control returns to the same student? Three coupled challenges arise: localizing the \emph{earliest} plausibly causal harmful boundary; comparing correction and student from the \emph{identical} environment state; and absorbing verified assistance into the student so deployment needs no teacher. We call this \textbf{verifier-mediated local teacher credit assignment}.

Our framework, \textbf{RE-0} (Figure~\ref{fig:mechanism}), solves it recursively. Given a failed trajectory of the current student $\pi_k$, a diagnostic agent localizes the earliest candidate harmful boundary and constructs repair contracts from the smallest scope; the teacher emits one bounded code macro-action within the contract, never taking over the remaining trajectory. RE-0 restores the environment to the snapshot saved at that boundary, executes the student baseline and the teacher correction under shared randomness, and admits the intervention only if the lower confidence bound of the paired advantage is positive; control then returns to the student, yielding a verified local assistant $\mu_k$ that intervenes only where the teacher has positive local credit. To convert this transient advantage into standalone capability, \textbf{RE-OPD} distills on the histories actually visited by $\mu_k$: the current student is the retention term, the contract-constrained teacher the local target, and the distillation weight is the verified counterfactual advantage; the next student then regenerates trajectories and supervision locations are re-localized on the new occupancy, giving the recursion $\pi_k\!\to\!\mu_k\!\to\!\pi_{k+1}$ whose per-round gain is lower-bounded by the verified intervention gain up to verification and projection error, with no global teacher dominance required. Our ablations locate the load-bearing component precisely: it is the verified admission of patches (the corpus), not the specific weighting---uniform SFT on the verified corpus is within trainer-seed variance of the weighted objective, while unverified corpora fail (composition and selection analyses in \S\ref{sec:exp}). On six long-horizon tasks and a physical robot platform, RE-0 lifts the base from 4--68\% to 62--100\% with 3--67 verified rows (details in \S\ref{sec:exp}).

Our contributions are: (1) verifier-mediated local teacher credit on student-induced histories; (2) \textbf{RE-OPD}, which recursively projects verified local improvement into a teacher-free student with a finite-sample improvement guarantee; and (3) empirical evidence across six long-horizon tasks and physical deployment that verified corpus selection---not weighting alone---is the load-bearing mechanism.

\section{Related Work}
\label{sec:related}

\paragraph{Code-as-Policy and embodied code agents.}
Large language models have been used to map natural-language tasks into embodied plans, callable skills, and executable robot programs: SayCan constrains planning with environment affordances, Inner Monologue folds environment feedback into closed-loop planning, and Code as Policies and ProgPrompt exploit code structure, control APIs, and program context to generate embodied policies \citep{ahn2022saycan,huang2022innermonologue,liang2023codeaspolicies,singh2023progprompt}. Recent work repairs code policies from vision and execution traces, or searches and evolves high-return programs with rollout evaluators \citep{lin2025embodiedcoder,sygkounas2026memento}. These methods focus on inference-time planning, program repair, or candidate search; RE-0 asks how corrections discovered \emph{during execution} become verifiable teacher credit amortized into a teacher-free student.

\paragraph{On-policy distillation and interactive imitation.}
Interactive imitation queries an expert on learner-induced states to mitigate covariate shift, with DAgger as the classic form \citep{ross2011dagger}; in autoregressive models, GKD distills on student-generated sequences and DistiLLM lowers its cost via adaptive off-policy reuse \citep{agarwal2024gkd,ko2024distillm}. For multi-turn agents, prefixes near the student occupancy can themselves make the teacher target unreliable, motivating reliability-aware replay approximations of OPD. RE-0 shares the goal of learning on the student distribution but never presumes teacher reliability on queried histories: it estimates the teacher's local counterfactual advantage against student behavior from the same simulation snapshot and projects only environment-verified positive-credit interventions.

\paragraph{Agentic correction, program repair, and verifiers.}
Self-Refine, Reflexion, and CRITIC revise generations with self-feedback, verbal memory, or external tools \citep{madaan2023selfrefine,shinn2023reflexion,gou2024critic}; in code generation, CodeRL scores programs from compilation and unit tests \citep{le2022coderl}. Such methods typically regenerate whole answers, select among candidates at the outcome level, or feed execution traces back as prompts. RE-0 instead localizes the \emph{earliest} harmful code boundary in the student's trajectory, compares the student baseline and the teacher correction from the \emph{identical} environment snapshot, and confines teacher control to a minimal verified macro-action after which the student regenerates the suffix; the verifier thus decides not only whether a candidate succeeds but whether the teacher's distribution earns local training credit.

\paragraph{Selective supervision and local credit assignment.}
Prior work selectively queries experts or weights demonstrations by uncertainty or quality \citep{zhang2017safedagger,kelly2019hgdagger,zhang2021confidenceaware}. RE-0 instead defines the credit object itself: the teacher's conditional behavior on a specific student history, evaluated by paired same-state counterfactual advantage.
\section{RE-0: Recursive Verified Policy Improvement}
\begin{figure}[t]
\centering
\includegraphics[width=\linewidth]{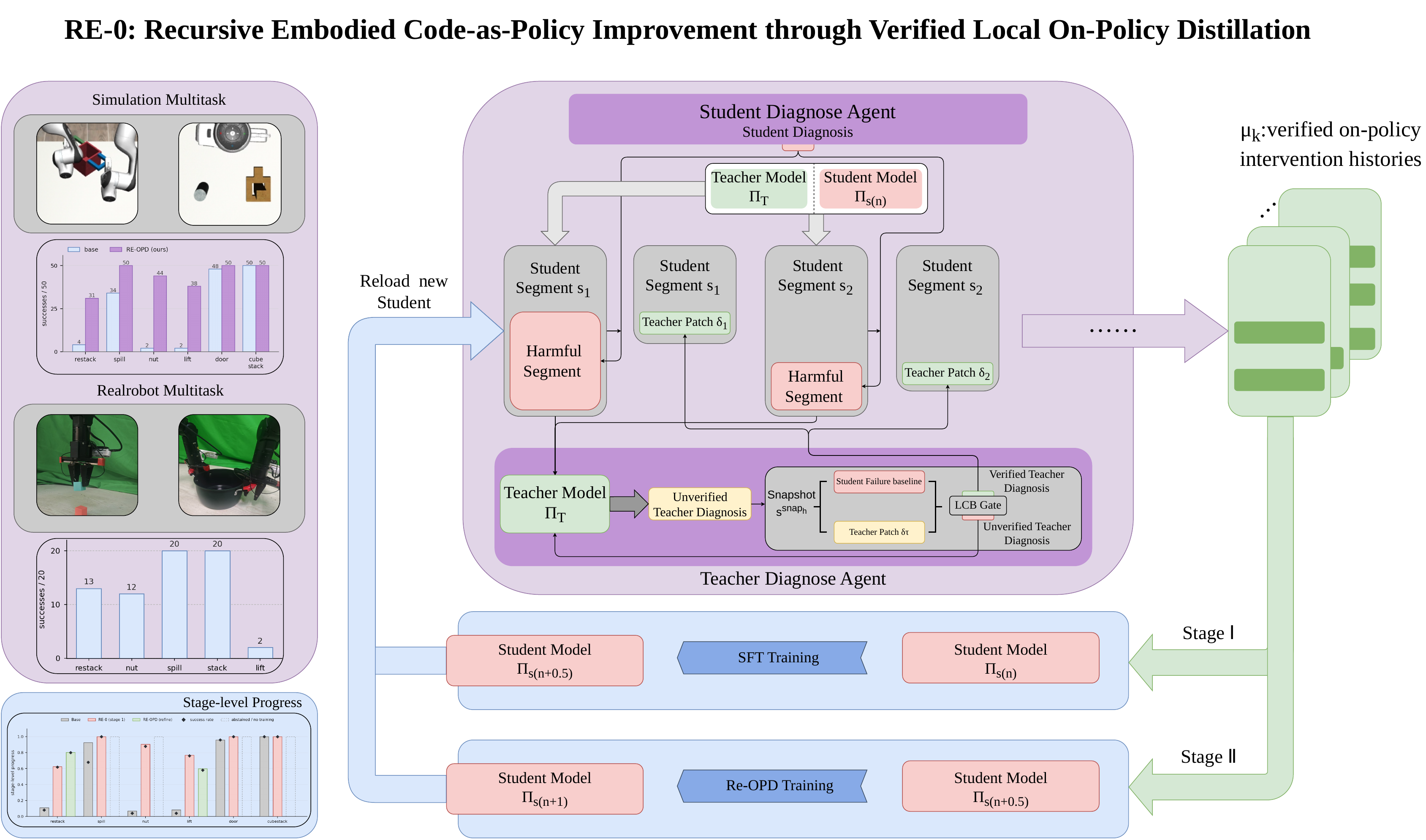}
\caption{\textbf{RE-0 overview.} RE-0 localizes harmful segments in the current student's trajectories, requests bounded teacher repairs, and verifies them by paired counterfactual rollout from the same snapshot (LCB gate). Verified interventions $\mu_n$ are distilled into the standalone student via cold-start SFT (Stage~I) and RE-OPD refinement (Stage~II); the student then regenerates trajectories, closing the loop. \textbf{Left:} simulation, real-robot, and stage-level results.}
\label{fig:mechanism}
\end{figure}

\label{sec:method}
\paragraph{Learning from a locally better but globally imperfect teacher.}
Standard on-policy distillation matches the teacher's conditional distribution on histories visited by the student. RE-0 assumes no global policy dominance between student and teacher: the teacher holds positive \emph{local} counterfactual advantage only on some student-induced histories, and supervision credit is assigned not to teacher trajectories as wholes but to the marginal improvement an intervention produces over the student's own behavior (Eq.~\ref{eq:delta-h}). The correction $\delta$ takes over only a minimal executable segment and immediately returns control to the current student $\pi_k$. A diagnostic agent localizes failure boundaries that may carry positive credit, an auditable counterfactual executor verifies the credit is real, and RE-OPD amortizes the verified improvement into the next independent student---verifier-mediated local credit assignment.

\subsection{Local Teacher Credit on Student-Induced Histories}
\label{sec:credit}
With $J(\pi)$ the expected discounted return, the policy-visible history at the $t$-th code decision point is $h_t=(o_0,c_0,\ldots,o_t)$, bounded code segments acting as macro-actions, $d_{\pi_k}^{t}$ the step-$t$ history distribution of the generation-$k$ student, and $s_t^{\mathrm{snap}}$ the verifier's policy-invisible snapshot. RE-0 asks whether the teacher decides better \emph{locally} on histories actually induced by the current student: for the teacher's local correction policy $\tau_k(\delta\mid x,h_t)$, whose corrections return control to $\pi_k$ after executing, the expected local credit is (formal definitions in Appendix~\ref{app:proofs})
\begin{equation}
\Delta_k(h_t)=\mathbb{E}_{\delta\sim\tau_k(\cdot\mid x,h_t)}\left[A_t^{\pi_k}(h_t,\delta)\right].
\label{eq:delta-h}
\end{equation}
Equation~\ref{eq:delta-h} captures the counterfactual marginal contribution of teacher intervention relative to the current student. If $\Delta_k(h_t)>0$, the teacher's local policy is better than continuing the student at $h_t$; if $\Delta_k(h_t)\le 0$, the teacher's supervision on that history must not be accepted automatically, regardless of its overall strength.

The ideal allocation---supervision only on $\mathcal{H}_k^{+}=\{(t,h_t):\Delta_k(h_t)>0\}$ under budget $B$, a fractional-knapsack LP (appendix)---is not directly observable: $\Delta_k(h)$ is hidden and the history space forbids exhaustive verification. RE-0 therefore uses diagnostics to localize executable failure boundaries that may carry positive local credit, then decides by counterfactual execution from identical states.

\begin{figure}[t]
\centering
\includegraphics[width=\linewidth]{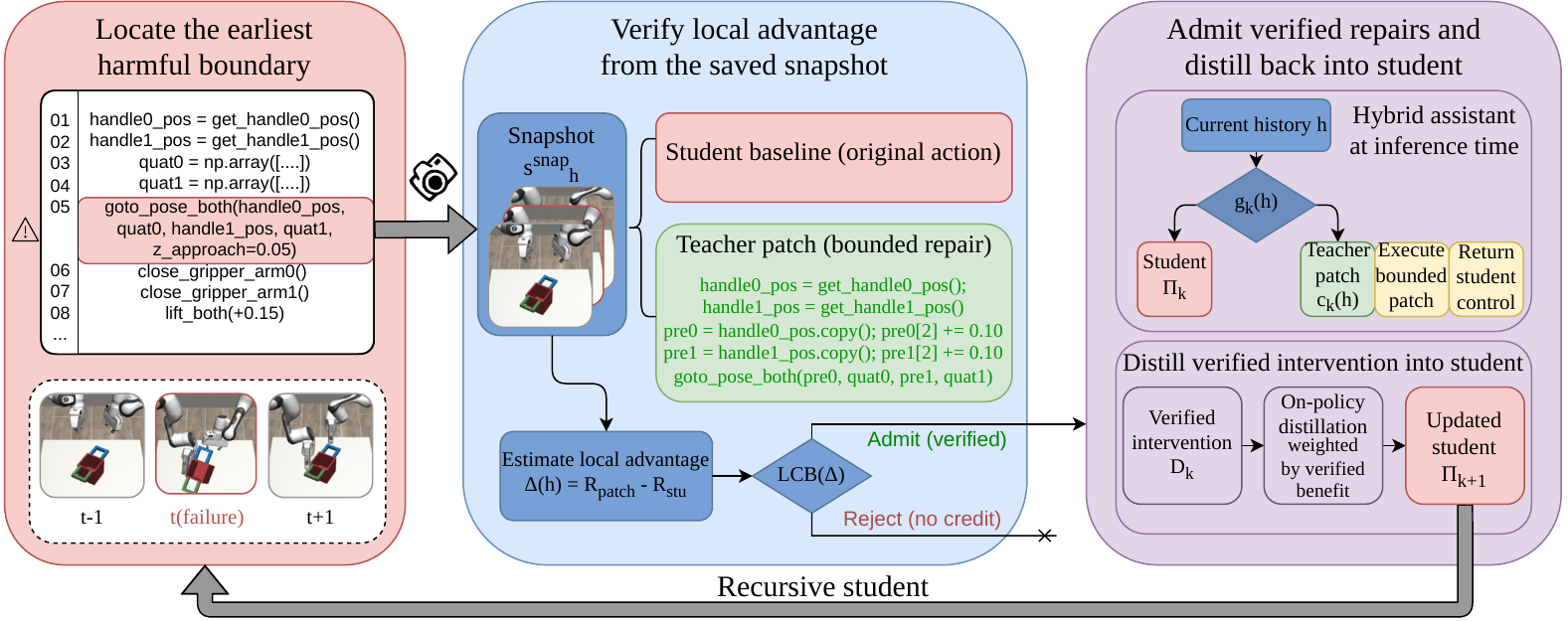}
\caption{\textbf{One recursion on a failing two-arm lift program.} \textbf{Locate:} the earliest harmful boundary $t$(failure). \textbf{Verify:} from the same snapshot, student baseline vs.\ teacher patch under shared randomness; admitted iff $\LCB>0$. \textbf{Distill:} verified repairs are credit-weighted into the standalone student; deployment uses no teacher calls.}
\label{fig:pipeline}
\end{figure}
\subsection{Verified Local Assistance}
\label{sec:assistance}
Equation~\ref{eq:delta-h} defines the teacher's true local credit, which cannot be observed directly. RE-0 approximates this ideal allocation with a two-stage mechanism: a diagnostic agent localizes the executable failure boundary that may carry positive local credit, and an environment verifier decides through same-state counterfactual execution whether the intervention yields a real gain. The former shrinks the candidate space; the latter decides whether teacher supervision enters the hybrid policy.

Given a student-executed trajectory $\zeta$, the student diagnostic agent outputs the earliest candidate failure history with structured evidence, $(\widehat{h},z)$: $z$ collects runtime exceptions, unrealized expected effects, or other structured failure evidence, and the execution system saves the full simulation snapshot $s_{\widehat{h}}^{\mathrm{snap}}$; if no executable, causally consistent boundary is found, no teacher call fires. A teacher diagnostic then judges whether the failure is locally repairable and emits repair contracts of monotonically increasing scope cost $\kappa$ (formally, $(\mathfrak{X}(\widehat{h},z),\,g_k^{\mathrm{diag}}(\widehat{h}))=\mathcal{A}_k^{T}(x,\widehat{h},z)$; definitions in the appendix). Each contract constrains the repair boundary, the allowed program symbols and APIs, and the maximum segment length. RE-0 verifies contracts in increasing $\kappa(\chi)$ and selects the smallest \emph{verified} repair scope---\emph{minimal necessary repair} means the smallest verified scope, not the shortest candidate string (selection rule and validity predicate in the appendix). If $g_k^{\mathrm{diag}}(\widehat{h})=0$ or no contract qualifies, no intervention fires; we write the selected contract teacher distribution as $\bar{\tau}_k^{\chi_k^{\star}}$.

For each teacher candidate, RE-0 restores the environment to $s_{\widehat{h}}^{\mathrm{snap}}$. Under shared subsequent randomness $\xi_1,\ldots,\xi_M$, the teacher correction and a student resample $c_m^{S}\sim\pi_k(\cdot\mid x,\widehat{h})$ are executed and control returns to the same student $\pi_k$ after the segment completes; the paired counterfactual advantage estimate $\widehat{A}_k(\widehat{h},\delta)$ is the $M$-average of the return difference $(\delta\!\to\!\pi_k)$ minus $(c_m^{S}\!\to\!\pi_k)$. RE-0 verifies the \emph{expected} credit of the contract teacher distribution: averaging $K$ candidates $\delta_j\overset{\mathrm{i.i.d.}}{\sim}\bar{\tau}_k(\cdot\mid x,\widehat{h})$ with a two-layer confidence radius controlling both teacher-candidate and rollout sampling error,
\begin{equation}
\Dhat_k(\widehat{h})=\tfrac{1}{K}\textstyle\sum_{j}\widehat{A}_k(\widehat{h},\delta_j),\qquad
\LCB_k(\widehat{h})=\Dhat_k(\widehat{h})-b_k(\widehat{h};K,M,\alpha).
\label{eq:lcb}
\end{equation} When verifying multiple histories we allocate $\alpha_{k,h}$ with $\sum_{k,h}\alpha_{k,h}\le\alpha$ for simultaneous coverage over the whole recursion; concentration bounds appear in the appendix. The environment gate is $g_k^{\mathrm{env}}(\widehat{h})=\mathbf{1}[\LCB_k(\widehat{h})>\epsilon]\cdot\mathbf{1}[\operatorname{Valid}(\chi_k^{\star},\bar{\tau}_k,s_{\widehat{h}}^{\mathrm{snap}})]$, with $\epsilon\ge0$ the minimum advantage threshold and $\operatorname{Valid}$ the execution-safety, program-context, and snapshot-identity predicate (appendix); candidate returns never post-select a single best candidate.

Combining the budget allocation $q_k$ (appendix), the final intervention probability is $g_k(h)=q_k(h)\,g_k^{\mathrm{diag}}(h)\,g_k^{\mathrm{env}}(h)$, giving the verified local hybrid policy
\begin{equation}
\mu_k(c\mid x,h)=\big(1-g_k(h)\big)\,\pi_k(c\mid x,h)+g_k(h)\,\bar{\tau}_k(c\mid x,h).
\label{eq:hybrid}
\end{equation}
From here on, the teacher policy in Eq.~\ref{eq:delta-h} is taken to be the same contract distribution $\bar{\tau}_k$. The teacher takes over only the current bounded repair segment; after it executes, the student regenerates subsequent code from the new observation rather than executing a suffix generated by the old policy.

\textbf{Remark (rewind semantics).} The pipeline runs two passes---a live attempt that localizes candidate boundaries, then verification from the stored snapshot---but both arms of every paired comparison start from the \emph{same} snapshot, so the first pass is a common data-collection cost that cancels in Eq.~\ref{eq:mu-bound}; $\mu_k$ is equivalently a one-pass policy of the \emph{augmented} process whose state includes the snapshot stack and the first-pass diagnostic evidence, and the bound applies to the post-rewind continuation return of that process, not to a snapshot-free online policy (appendix).

By the history--state performance-difference lemma, $J(\mu_k)-J(\pi_k)=\sum_t\gamma^t\mathbb{E}_{h\sim d_{\mu_k}^t}[g_k(h)\Delta_k(h)]$. Define the verifiable intervention gain $\widehat{\mathcal{G}}_k=\sum_t\gamma^t\mathbb{E}_{h\sim d_{\mu_k}^t}[g_k(h)[\Dhat_k(h)]_+]$ and the verification error budget $\mathcal{E}_{\mathrm{verify},k}=\sum_t\gamma^t\mathbb{E}_{h\sim d_{\mu_k}^t}[g_k(h)b_k(h)]+\mathcal{E}_{\mathrm{sys},k}$, where $\mathcal{E}_{\mathrm{sys},k}$ collects system approximation errors such as snapshot restoration and execution identity. Whenever returns are bounded, contract supports are valid, and the simultaneous-coverage condition of Eq.~\ref{eq:lcb} holds, with probability at least $1-\alpha$,
\begin{equation}
J(\mu_k)-J(\pi_k)\;\ge\;\widehat{\mathcal{G}}_k-\mathcal{E}_{\mathrm{verify},k}.
\label{eq:mu-bound}
\end{equation}
Thus RE-0 converts verified positive local credit into expected performance gain of the hybrid policy without requiring global teacher dominance.

\subsection{RE-OPD: Amortizing Local Credit into the Student}
\label{sec:reopd}
The verified local hybrid $\mu_k$ obtains higher expected return through local teacher interventions, but executing it still requires teacher calls and environment verification. RE-OPD amortizes these transient local advantages into student parameters, so that $\pi_{k+1}$ reproduces the improved behavior without teacher access.

Recall the hybrid of Eq.~\ref{eq:hybrid}. The ideal projection finds the independent student closest to $\mu_k$ on the histories $\mu_k$ actually visits, $\pi_{k+1}=\arg\min_{\pi_\theta}\sum_t\gamma^t\mathbb{E}_{h\sim d_{\mu_k}^t}[D_{\mathrm{KL}}(\mu_k(\cdot\mid x,h)\Vert\pi_\theta(\cdot\mid x,h))]$. Training histories cover both the failure states the student would visit and the new states induced by the student after teacher correction---not fixed teacher trajectories, nor mechanically stitched student suffixes.

Because $\mu_k$ is itself the $g_k$-mixture of $\pi_k$ and $\bar{\tau}_k$ (Eq.~\ref{eq:hybrid}), substituting into the KL and dropping parameter-independent entropy terms turns the ideal projection \emph{exactly} into a $g_k$-weighted sum of a student-retention term and a teacher-projection term; RE-OPD re-weights teacher supervision by verified local credit:
\begin{equation}
\mathcal{L}_{\mathrm{RE\text{-}OPD}}(\theta)=\sum_t\gamma^{t}\,\mathbb{E}_{h\sim d_{\mu_k}^{t}}\Big[(1-g_k(h))\,D_{\mathrm{KL}}\big(\pi_k\,\Vert\,\pi_\theta\big)+\lambda\, g_k(h)\, w_k(h)\,D_{\mathrm{KL}}\big(\bar{\tau}_k\,\Vert\,\pi_\theta\big)\Big],
\label{eq:loss}
\end{equation}
with all conditional distributions conditioned on $(x,h)$, $\lambda$ controlling local distillation strength, and the local credit weight $w_k(h)=[\LCB_k(h)]_+$. Equation~\ref{eq:loss} is a credit-weighted surrogate inspired by the exact mixture projection, not a strict upper bound of the ideal projection for arbitrary $\lambda w_k(h)$; through $g_k(h)w_k(h)$ it decides \emph{how much} teacher supervision earns training credit \emph{where}.

Restricted estimators implement the same projection (derivation in the appendix): Top-$K$ token probabilities give a truncated sparse estimate, and a black-box teacher degenerates to the trajectory-level hard estimator $\ell_{\mathrm{hard}}(h,\delta)=-\lambda g_k(h)w_k(h)\log\pi_\theta(\delta\mid x,h)$ used throughout our experiments.

Optimizing Eq.~\ref{eq:loss} controls the occupancy-weighted projection error $\varepsilon_{\mathrm{proj},k}$ (definition in the appendix), for which $[J(\mu_k)-J(\pi_{k+1})]_+\le C_{T,\gamma,R}\sqrt{\varepsilon_{\mathrm{proj},k}}$ with $C_{T,\gamma,R}$ depending only on horizon, discount, and return bound (proof in the appendix); combining this with Eq.~\ref{eq:mu-bound} yields the core result.

\begin{theorem}[Recursive verified policy improvement]
\label{thm:main}
If returns are bounded, contracts are valid, the counterfactual verifier satisfies the simultaneous-coverage condition of Eq.~\ref{eq:lcb}, and the occupancy-weighted projection error of RE-OPD is bounded, then with probability at least $1-\alpha$,
\begin{equation}
J(\pi_{k+1})-J(\pi_k)\;\ge\;\widehat{\mathcal{G}}_k-\mathcal{E}_{\mathrm{verify},k}-C_{T,\gamma,R}\sqrt{\varepsilon_{\mathrm{proj},k}}.
\label{eq:thm}
\end{equation}
Hence $J(\pi_{k+1})>J(\pi_k)$ whenever $\widehat{\mathcal{G}}_k>\mathcal{E}_{\mathrm{verify},k}+C_{T,\gamma,R}\sqrt{\varepsilon_{\mathrm{proj},k}}$.
\end{theorem}

After each projection, RE-0 regenerates trajectories with $\pi_{k+1}$ from the raw task input and re-allocates local teacher credit on the new occupancy.

\begin{remark}[Scope of the guarantee]
\label{rem:scope}
$\widehat{\mathcal{G}}_k$ and $\mathcal{E}_{\mathrm{verify},k}$ are measured per event by the audit and the admission rule; the projection term is not estimated a priori but guarded by the outer acceptance criterion (Remark~\ref{rem:twoc}), which retains a generation only if its measured performance does not regress---verified at the credit level, guarded by measurement at the policy level (per-round sign accounting in Appendix~\ref{app:robust}).
\end{remark}

\subsection{Cold Start and Recursive Refinement}
\label{sec:coldrefine}
The construction assumes nothing about \emph{who} occupies the policy seat: replacing $d_{\pi_k}$ by $d_\sigma$ for any seat policy $\sigma$ defines $\mu^{\sigma}=\mathcal{I}(\sigma)$, and under the simultaneous coverage of Eq.~\ref{eq:lcb} the guarantee of Eq.~\ref{eq:mu-bound} transfers verbatim, with $\mu^\sigma=\sigma$ when the gate is identically zero---no global capability relation, no training (Proposition~1, appendix).

\paragraph{RE-OPD-cold.}
Taking $\sigma=\tau$ puts the teacher in the seat: its own failure histories are likewise diagnosed, repaired under progressive contracts, and verified by the same gate, and the verified-completed trajectories form the cold corpus $\mathcal{D}_0$ on which full-trajectory supervision yields $\pi_1$. $\mathcal{D}_0$'s quality is guaranteed by Proposition~1, not by global teacher reliability---even a teacher unable to complete the task contributes trajectories inheriting verified self-improvement credit, and $\pi_1$ inherits the lower bound of Theorem~\ref{thm:main} with the teacher in the seat (Corollary~1, appendix; the cold surrogate is completion-conditioned there).

\paragraph{RE-OPD-refine.}
Later generations put the student in the seat, $\pi_{k+1}=\mathcal{P}_\Theta(\mathcal{I}(\pi_k))$; the projections from the teacher and student seats are \textbf{RE-OPD-cold} and \textbf{RE-OPD-refine}, sharing all machinery and differing only in the occupancy source of the corpus. Termination is structural: if $\LCB_k(h)\le\epsilon$ for almost all $h\in\operatorname{supp}(d_{\pi_k})$, the gate closes, Eq.~\ref{eq:loss} degenerates to the retention term, and the recursion safely abstains (Corollary~2, appendix)---triggered by exhausted headroom or by teacher effect below the verification noise floor $b_k$.

\begin{remark}[Deployment criterion and two conditions for useful refinement]
\label{rem:twoc}
Given threshold $\rho$, halt if $\widehat{J}(\pi_1)\ge\rho$ and otherwise refine, re-checking after each round. Refinement is useful only when \emph{headroom} $\eta$ (student below $\rho$) and \emph{verifiable teacher effect} $\nu$ (admitted credit mass above the noise floor) hold \emph{jointly}; either alone is insufficient, as the lift experiment in \S\ref{sec:main-results} shows empirically.
\end{remark}

\section{Experiments}
\label{sec:exp}

\subsection{Experimental Setup}
\label{sec:setup}
\paragraph{Tasks and environments.}
We evaluate on six long-horizon robosuite Code-as-Policy tasks---\emph{cube restack}, \emph{spill wipe}, \emph{nut assembly}, \emph{two-arm lift}, \emph{cube stack}, and \emph{door opening}---covering single/dual arms, rigid/articulated objects, and coarse/fine manipulation (details in the appendix). Base-model success spans 4\% to 100\% (restack 8\%, lift 4\%, nut 4\%, spill 68\%, door 96\%, cubestack 100\%), a full difficulty spectrum; each task keeps a held-out seed set disjoint from training collection.

\paragraph{Models, seats, collection, and training.}
The student is Qwen3.8-27B \citep{qwen2026qwen38} fine-tuned with LoRA; the teacher and both diagnostic agents are the qwen3.8-max API. The teacher is black-box (no conditional distributions), so all projections use the trajectory-level hard estimator; Eq.~\ref{eq:loss} thus becomes a supervised loss weighted by verified credit, with the same optimizer family as SFT---the difference is not the loss form but \emph{what enters it}: supervision locations, per-event weights $g_kw_k$, and the retention term $(1-g_k)$ are decided by environment verification, not teacher sampling. Baselines (plain SFT, RFT/STaR, GRPO, OPD without verification) share the optimizer, loss family, and evaluation protocol, differing only in corpus provenance and quality assurance. Collection sends fresh-rollout failures (cycles the student already solves are skipped) through the mechanism of \S\ref{sec:assistance} with LCB admission per Eq.~\ref{eq:lcb}; the same pipeline on the teacher seat and the student seat yields the \textbf{RE-OPD-cold} and \textbf{RE-OPD-refine} corpora (hyperparameters in the appendix). \emph{Abstention rule}: tasks already at $\ge$80\% success do not enter refinement---the practical instance of Corollary~2---and are marked as abstained. The execution system saves full simulation snapshots at code decision points; diagnostic evidence $z$ includes 1--2 perception-derived physical quantities per task (\S\ref{sec:ablations}).

\paragraph{Evaluation protocol.}
All methods use 50 paired held-out seeds per task (Wilson 95\% CIs in Appendix~\ref{app:robust}; paired per-seed outcomes for McNemar tests ship with the released artifacts) and share identical training and evaluation infrastructure apart from the variable under study. One protocol caveat: stage-entry ($\rho$) decisions reused the 50-seed evaluation cohort; collection seeds are always disjoint, and re-evaluation of all base and cold checkpoints on a disjoint, never-used 50-seed cohort reproduces the reported conclusions within binomial variation (Appendix~\ref{app:trseed}).

\subsection{Main Results}
\label{sec:main-results}
Table~\ref{tab:main} reports held-out success on 50 seeds across the six tasks.

\begin{table}[t]
\centering
\footnotesize
\begin{tabular}{lccccc}
\toprule
Task & Base & Flash-API & RE-OPD-cold ($n$ rows) & RE-OPD-refine \\
\midrule
restack & 4/50$^{a}$ & 2 & \textbf{31} (3) & \textbf{40}$^{b}$ \\
spill & 34 & 33 & \textbf{50} (67) & ---$^{c}$ \\
nut & 2 & 26 & \textbf{44} (27) & --- \\
lift & 2 & 6 & \textbf{38} (7) & 29$^{d}$ \\
door & 48 & --- & \textbf{50} (40) & --- \\
cubestack & 50 & 50 & 50 (0, no-train)$^{e}$ & --- \\
\bottomrule
\end{tabular}
\caption{\textbf{Main results} (successes / 50 held-out seeds). Flash-API: same teacher, single-shot without diagnose--repair--verify. $n$: verified training rows. ``---'': criterion abstained, not unmeasured. Trainer-seed replication (three seeds, Appendix~\ref{app:trseed}): restack 31/44/36, spill 50/33/33, lift 38/33/35, nut 44/40/43, door 50/50/50.}
\label{tab:main}
\end{table}

\paragraph{Sample efficiency of verified data.}
RE-OPD-cold lifts the base from 4--68\% to 62--100\% with 3--67 verified rows (restack 62\% with 3 rows, lift 76\% with 7, nut 88\% with 27)---improvement is not monotone in row count; the decisive factor is the verified credit each row carries (Corollary~1), not corpus size. The same teacher's single-shot calls reach only 4\% on restack while solving nut (52\%): without the diagnose--repair--verify loop, teacher capability does not land on hard tasks. \emph{When no teacher is needed} is decided by the framework itself: cubestack is saturated and exempted; door goes 96\%$\to$100\% with 40 rows.

\paragraph{Adaptive criterion.}
Refinement happens only where the target is unmet: restack 62\%$<\rho{=}80\%$ refines to 80\%; spill is already at 100\%, its admission is empty, and the framework correctly abstains---a runtime instance of Corollary~2. Admissions decay with occupancy (6/6 at the $\pi_1$ round vs.\ 4/15 at $\pi_2$; Figure~\ref{fig:admission}b; per-round accounting in the appendix). The lift row exercises the outer acceptance test (Remark~\ref{rem:scope}): 76\%$<\rho$ but $\pi_2$ admits one event ($w{=}0.069$)---thin verified credit; the projection regresses to 58\% with branch collapse (Appendix~\ref{app:robust}, note~(d); the regression replicates at training seed 2, 76\%$\to$46\%), and the criterion halts the recursion, disclosing the regressed checkpoint rather than propagating it: $\eta$ without $\nu$ yields a harmful projection, and it is the stage-entry measurement that catches it.

\paragraph{Failure of pooling.}
Naively pooling verified data across tasks causes severe interference despite larger corpus size (Table~\ref{tab:pooling}, appendix), consistent with the allocation failure mode in \S\ref{sec:credit}.

\subsection{Verification and Credit Analysis}
\label{sec:analysis}

\paragraph{The admission ledger measures teacher effect.}
The same gate, applied to different teachers, yields a monotone ladder (appendix, Table~\ref{tab:ladder}): 0 admissions for student self-play (34 cycles) and for narrow local experts (32 events; $\sim$10\% rescue ceiling, deterministic contract-format failures), versus 4 admissions on 13 events for the strong API teacher (rescue rates 19--47\%). These gradings precede any training: the verifier directly measures the teacher effect $\nu$ of Remark~\ref{rem:twoc}---not ``who is stronger,'' but ``who holds verifiable local credit on the student's failure history.''

\paragraph{Strong separation makes the strict gate free; the gate is two-sided.}
Admitted events cluster at $\Dhat\in[0.19,0.44]$ while abstained ones satisfy $|\Dhat|\le0.04$ (Figure~\ref{fig:admission}c); sweeping the threshold from 0.24 to 0.01 leaves the admitted set unchanged (Figure~\ref{fig:admission}d)---the LCB gate's statistical strictness costs no sensitivity. The verifier also rejects harm: one API patch dropped the baseline 8/16$\to$1/16 ($\Dhat=-0.44$) and was rejected. Credit assignment is a decision on the causal \emph{sign} of an intervention, not one-directional imitation.

\begin{figure}[t]
\centering
\includegraphics[width=\linewidth]{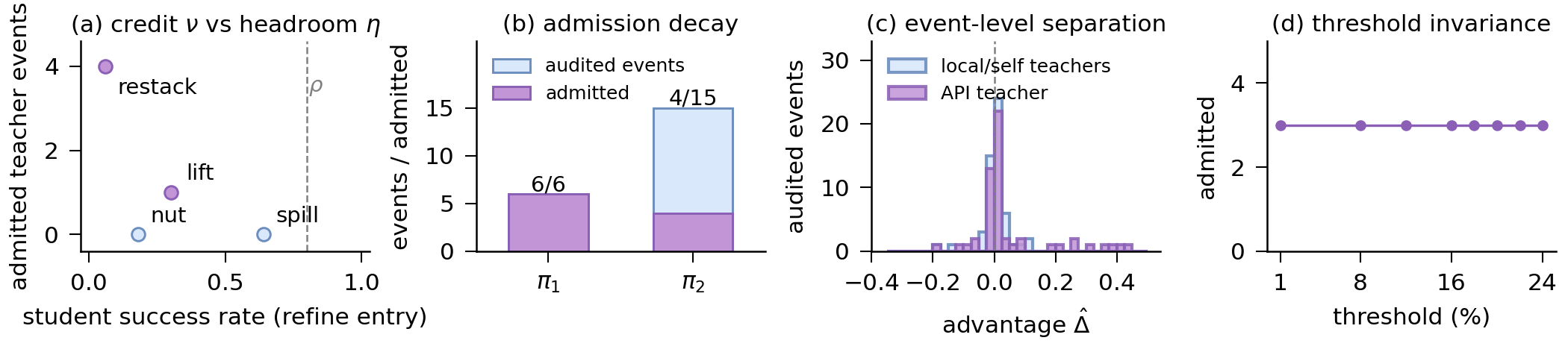}
\caption{\textbf{Credit analysis.} (a) Headroom alone does not imply credit: all tasks sit below $\rho{=}0.80$ yet admissions range 0--4. (b) Admission decay across rounds. (c) $\Dhat$ separation: admitted vs.\ abstained over 110 events. (d) Threshold sweep: the admitted set is invariant.}
\label{fig:admission}
\end{figure}

\begin{figure}[!t]
\centering
\includegraphics[width=\linewidth]{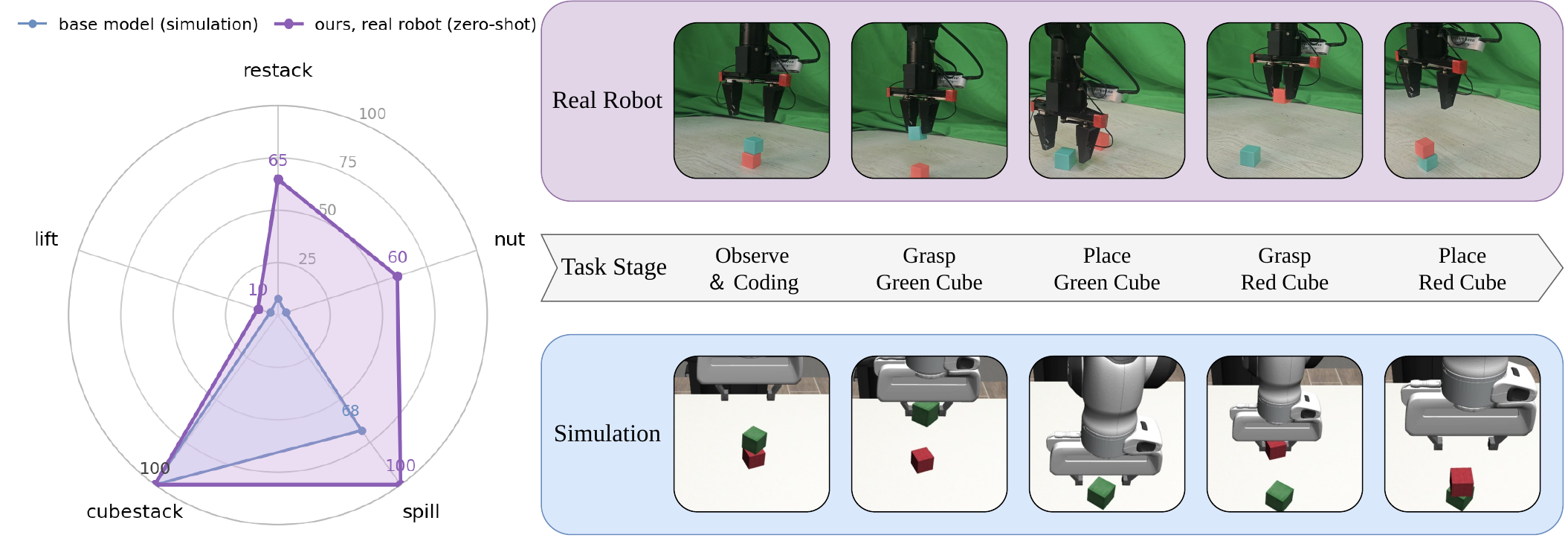}
\caption{\textbf{Zero-shot sim-to-real transfer.} Real-robot execution of the simulation-trained policies on the dual-arm AgileX PiPER platform: task keyframes (left), hardware success counts on 20 trials per task (middle), and hardware success rates against the simulation base model (right). The deployed policy code is identical to the simulation policy---no hardware-specific edits.}
\label{fig:simreal}
\end{figure}

\paragraph{Passing the gate $\neq$ policy flip.}
Weight-scale and normalization ablations cannot recover the gain of the verified corpus and can induce catastrophic forgetting (Appendix~\ref{app:selection}), showing that verification determines corpus quality while weighting alone cannot substitute for composition-preserving training. This isolates the load-bearing variable: the gate's product is the corpus.

\paragraph{GRPO control.}
With terminal rewards in the same environments, GRPO \citep{shao2024deepseekmath} undergoes entropy collapse (0.30$\to$0.002) and transfers none of it: 0--1/50 on the non-saturated tasks against base 8/68/4\%, retaining only the saturated task (details and tuning notes in the appendix). Sparse terminal rewards cannot localize the harmful decision inside a long program, in mechanistic contrast to locate--verify--weight.

\label{pg:lA}\subsection{Ablations}
\label{sec:ablations}
\paragraph{Selection-rule ablation: the gate is the source of credit quality.}
Holding the candidate journal and training recipe fixed, LCB admission outperforms unfiltered and sign-only selection on lift across three trainer seeds (Table~\ref{tab:selection}; Appendix~\ref{app:selection}).

\begin{table}[t]
\centering\small
\begin{tabular}{llcccc}
\toprule
Selection rule (lift journal) & Rows & \multicolumn{3}{c}{Successes / 50} & Mean \\
 & & $s_1$ & $s_2$ & $s_3$ & \\
\midrule
None (all candidates) & 25 & 6 & 2 & 3 & 3.7 \\
Sign ($\Dhat>0$) & 7 & 5 & 4 & 4 & 4.3 \\
LCB (paired, Eq.~\ref{eq:lcb}) & 3 & \textbf{11} & \textbf{5} & \textbf{13} & \textbf{9.7} \\
Outcome (teacher programs) & 27 & 5 & 9 & 4 & 6.0 \\
\bottomrule
\end{tabular}
\caption{\textbf{Selection-rule ablation} (lift; three trainer seeds; uniform weights). LCB dominates unfiltered and sign rules in every seed pairing.}
\label{tab:selection}
\end{table}

\paragraph{Corpus composition is decisive.} On restack, success-filtered teacher demonstrations remain at base (3/50), whereas three verified event patches reach 46/50. Outcome filtering selects the identical corpus on this task, so the discriminating admission-rule comparison is Table~\ref{tab:selection} (Appendix~\ref{app:selection}).

\paragraph{Objective-composition quartet; teacher necessity.}
Replacing only the objective composition (restack; separate 50-seed cohort; armwise scores, definitions, and per-cohort variance in the appendix) confirms the decomposition of Eq.~\ref{eq:loss}: stripping the student prefix (patch-only) or the retention term (weighted) collapses the gain that the full objective and uniform credit largely retain. Student self-bootstrap yields 0 admissions on four tasks, echoing \S\ref{sec:analysis}.

\label{pg:lG}\subsection{Generalization and Real-Robot Deployment}
\label{sec:generalization}
\paragraph{Transfer and deployment.}
Cross-task evaluation shows selective rather than universal transfer: restack training improves nut while spill's wiping dynamics do not carry, and pooled multi-task training exhibits substantial interference (Appendix~\ref{app:transfer}). We additionally deploy the distilled, teacher-free policies zero-shot, without hardware-specific policy edits, on a dual-arm AgileX PiPER platform with SAM3-based open-vocabulary perception, obtaining 13/20 restack, 12/20 nut, 20/20 spill, 20/20 cubestack, and 2/20 lift successes (Figure~\ref{fig:simreal}, Appendix~\ref{app:realrobot}). Verification is used only during simulation-side collection; hardware deployment makes no teacher calls.

\label{pg:lC}\section{Conclusion}
\label{sec:conclusion}
We recast self-improvement of embodied Code-as-Policy agents as a \textbf{locate--verify--weight} recursion: teacher interventions are minimal patches, verified event-wise by same-snapshot counterfactual experiments, admitted only with finite-sample confidence, and returned to the student after a single bounded macro-action. Across six tasks and physical deployment, 3--67 verified rows lift the base from 4--68\% to 62--100\%, while the same admission ledger doubles as a measurement instrument driving self-termination. The central empirical finding is that verified corpus selection, not weighting alone, is the load-bearing mechanism---confirmed by selection ablations across three trainer seeds (Table~\ref{tab:selection}); making the ideal allocation of supervision under composition constraints explicit is the natural next step.

\clearpage
\label{pg:lL}\label{page:ai}
\subsection*{AI use statement}
\paragraph{Roles of AI systems.} All policies evaluated in this work are LLM-based code-as-policy agents; they constitute the \emph{object of study} and are reported as experimental setup, not as AI assistance.
\paragraph{Required disclosure.} We used generative AI tools to: (i) design and provide feedback on research methodology and experiments (experiment prioritization and ablation design were discussed with an AI assistant; all protocol parameters were fixed in advance by the authors and never altered); (ii) implement methods (operational scripts for training, evaluation, and data assembly were drafted with AI assistance and executed under author supervision; the core environment and pipeline code is the authors' own); (iii) interpret results (log analysis and attribution of negative results were AI-assisted and re-verified by the authors against raw artifacts); and (iv) assist with translation of the manuscript draft. We did \emph{not} use generative AI to generate synthetic datasets, to formulate or prove mathematical claims, or to produce any experimental data: every number in this paper is programmatically extracted from logged evaluation runs of the described pipeline.
\paragraph{Recommended disclosure.} We additionally used generative AI tools to draft parts of the paper, edit text for readability, create scientific figures (all figure scripts and underlying event data will be released), suggest paper structure, and format references.
\paragraph{Responsibility.} All AI-assisted code was executed and tested by the authors, all reported numbers were cross-checked against raw logs, and all AI-assisted text was reviewed and revised by the authors. We take full responsibility for the final content of this work.

\subsection*{Reproducibility statement}
We will release the complete pipeline and all artifacts: (i) environment configurations for all six robosuite tasks; (ii) the full admission journals---110 audited events with paired counterfactual advantages, lower confidence bounds, and assigned training weights; (iii) every verified training parquet with row-level provenance to admission events; (iv) raw evaluation JSONs underlying all reported evaluation tables and ablations, together with the admission journals underlying the analysis figures; and (v) training configurations. All numbers in the paper are traceable to these artifacts via a release manifest.

\label{page:refs}
\bibliography{refs}
\bibliographystyle{iclr2027_conference}

\appendix
\numberwithin{equation}{section}

\section{Limitations}
(1) The allocation layer is not closed: $w\equiv\mathrm{LCB}$ suppresses patches to 1--1.5\% weight, $\Sigma w\to1$ normalization fails to flip and destroys composition, and unprotected distillation forgets; converting verified credit into policy change needs an allocation theory under composition constraints (we contribute a stress test and the $\eta/\nu$ criteria). The teacher is black-box, so deployed distillation uses the trajectory-level hard estimator; the credit weighting did not measurably outperform uniform weighting on the same verified corpus---we claim verified data selection, not the weighting scheme, as the empirical mechanism. (2) Teacher necessity is calibrated against an API frontier model, and paired-counterfactual cost grows linearly with audited events. (3) Conclusions rest on six robosuite families and one physical platform; transfer is selective. (4) The $\pi_{1.5}$ row carries a collection-seat artifact; ablation pools differ from the main-table protocol; the normalization arm is a single run, and selection-ablation arms are single runs per seed (three seeds on lift). (5) The handback design (control returning to the student immediately after each verified patch, \S\ref{sec:assistance}) rests on design argument, and RFT/STaR remains inferred from the gate-off ablation and the self-imitation arm. (6) Stage-entry ($\rho$) decisions reused the evaluation cohort as validation (validation and test are not fully separated for the reported stage-entry decisions); a disjoint never-used cohort re-evaluation reproduces the base/cold conclusions (Appendix~\ref{app:trseed}), and the released protocol separates them prospectively. The reuse is consequential only for thin margins: retained generations cleared $+18$ to $+84$ points and abstentions sat $16$--$20$ points above $\rho$; the one thin entry (lift, 76\% against $\rho{=}80\%$) is precisely the case the criterion re-measured and rejected (\S\ref{sec:main-results}).

\section{Proofs and statistical guarantees}
\label{app:proofs}

Throughout, $\bar R := R\,\frac{1-\gamma^{T}}{1-\gamma}$ bounds any trajectory return ($T$-step horizon, $|r_t|\le R$), and $S_\gamma:=\sum_{t=0}^{T-1}\gamma^{t}$. Snapshot-restoration and execution-identity discrepancies are booked into $\mathcal{E}_{\mathrm{sys},k}$.

\paragraph{Formal definitions (\S\ref{sec:credit}).}
\begin{itemize}\setlength\itemsep{2pt}
\item \emph{Return.} $J(\pi)=\mathbb{E}_{x\sim\rho,\zeta\sim P^\pi}\big[\sum_{t=0}^{T-1}\gamma^t r_t\big]$.
\item \emph{Histories.} At the $t$-th code decision point, $h_t=(o_0,c_0,o_1,c_1,\ldots,c_{t-1},o_t)$ with bounded segments $c_t\in\mathcal{C}$ treated as macro-actions; the occupancy is $d_{\pi_k}^{t}(h_t\mid x)=\Pr_{\pi_k}(H_t{=}h_t\mid x)$.
\item \emph{Teacher policy and action values.} The teacher's local policy $\tau_k(\delta\mid x,h_t)$ proposes a bounded correction whose suffix is regenerated by $\pi_k$;
\begin{equation}
Q_t^{\pi_k}(h_t,\delta)=\mathbb{E}\Big[\sum_{s\ge t}\gamma^{s-t}r_s\,\Big|\,H_t{=}h_t,\,C_t{=}\delta,\,C_{>t}{\sim}\pi_k\Big],
\end{equation}
with $V_t^{\pi_k}(h_t)=\mathbb{E}_{c\sim\pi_k}[Q_t^{\pi_k}(h_t,c)]$, $A_t^{\pi_k}=Q-V$, and $\Delta_k(h_t)=\mathbb{E}_{\delta\sim\tau_k}[A_t^{\pi_k}(h_t,\delta)]$.
\item \emph{Contract selection.} $\chi_k^{\star}(\widehat h)=\arg\min_{\chi\in\mathfrak X(\widehat h,z)}\kappa(\chi)$ subject to $\widehat{\mathrm{LCB}}_k(\widehat h;\chi)>\epsilon$ and $\operatorname{Valid}(\chi,\bar\tau_k^\chi,s^{\mathrm{snap}}_{\widehat h})=1$, applied per $\bar\tau_k^\chi$ during selection, where $\operatorname{Valid}$ requires that the contract-constrained policy support meets execution-safety and program-context constraints and that snapshot-restoration and execution-identity checks agree; with constrained decoding, support constraints are realized implicitly by the decoder.
\end{itemize}

\paragraph{Two-layer confidence radius.}
Fix a history $\widehat h$ and let teacher candidates $\delta_1,\dots,\delta_K\overset{\mathrm{iid}}{\sim}\bar\tau_k(\cdot\mid x,\widehat h)$ with paired estimates $\widehat A_k(\widehat h,\delta_j)=\frac1M\sum_{m=1}^M[G^{T}_{jm}-G^{S}_{jm}]$ under shared randomness $\xi_m$, and $\widehat\Delta_k(\widehat h)=\frac1K\sum_j \widehat A_k(\widehat h,\delta_j)$. Since the paired differences lie in $[-\bar R,\bar R]$ with common noise eliminated by pairing, Hoeffding applied to the inner ($M$-pair) and outer ($K$-candidate) layers gives, for any $\alpha_1,\alpha_2\in(0,1)$, with probability at least $1-(\alpha_1+\alpha_2)$,
\begin{equation}
\big|\widehat\Delta_k(\widehat h)-\Delta_k(h)\big|\;\le\;\underbrace{\bar R\sqrt{\tfrac{2\ln(2/\alpha_1)}{K}}+\bar R\sqrt{\tfrac{2\ln(2/\alpha_2)}{M}}}_{=\,b_k(\widehat h;K,M,\alpha)}.
\label{eq:app-radius}
\end{equation}

\paragraph{Simultaneous coverage and selection.}
Allocate levels $\alpha_{k,h}$ with $\sum_{k,h}\alpha_{k,h}\le\alpha$ across all histories verified in generation $k$; by the union bound the event $\mathcal{E}_{\mathrm{cov}}:=\{\forall (k,h): \Delta_k(h)\ge\widehat\Delta_k(h)-b_k(h;K,M,\alpha_{k,h})\}$ holds with probability at least $1-\alpha$. The admission gate $g_k$ is measurable in $\widehat\Delta_k$ (data-dependent selection), but $\mathcal{E}_{\mathrm{cov}}$ holds uniformly over \emph{all} histories, so selection does not invalidate the per-history bound---this is why simultaneous rather than pointwise confidence bounds are used.

\paragraph{Implementation note (analytical vs.\ operational bound).}
The Hoeffding radius above is the analytical form used in the theorem statements. The deployed verifier (\texttt{anytime\_valid\_betting\_mixture\_v1}) uses an anytime-valid e-process over the same paired differences (success and normalized-reward channels in parallel), which enjoys the same coverage guarantee with a tighter radius; the theorem statements are unaffected by this choice. Admission used $\epsilon=0$ and $\alpha=0.05$; each audited event starts at $M{=}8$ shared-randomness pairs and escalates ($12/16/24/32/48$) when evidence is insufficient; candidates are drawn under top-$k$ constrained decoding within the selected contract. At the base pairing, one audited event costs up to $(K{+}1)\times M$ suffix rollouts; the full ledger contains 110 audited events.

\paragraph{History--state performance-difference lemma.}
For any policies $\pi,\pi'$,
\begin{equation}
J(\pi')-J(\pi)=\sum_{t=0}^{T-1}\gamma^t\,\mathbb E_{h\sim d_{\pi'}^{t}}\big[\mathbb E_{c\sim\pi'(\cdot\mid x,h)}[Q_t^{\pi}(h,c)]-V_t^{\pi}(h)\big].
\label{eq:app-pdl}
\end{equation}
\emph{Proof.} Induction on $T$. For $T{=}1$ both sides equal $\mathbb E_{h\sim d^0}\mathbb E_{c\sim\pi'}[r(h,c)]-V^\pi$. For the step, decompose $Q_T^\pi(h,c)=r(h,c)+\gamma\,\mathbb E_{h'}V_{T-1}^{\pi}(h')$, difference the first-step expectations of $\pi'$ and $\pi$, expand the continuation of $\pi'$ along $d_{\pi'}^{t}$, and apply the hypothesis. Only the Markov property of the macro-action history process is used. $\square$

\paragraph{Execution-time identity.} Under the mixture of Eq.~\ref{eq:hybrid},
\begin{equation}
J(\mu_k)-J(\pi_k)=\sum_{t}\gamma^t\mathbb E_{h\sim d_{\mu_k}^{t}}\big[g_k(h)\,\Delta_k(h)\big].
\label{eq:app-identity}
\end{equation}
\emph{Proof.} Take $\pi'=\mu_k$ in the lemma; per history, the student component $(1-g_k)\pi_k$ cancels against $V_t^{\pi_k}(h)$ and the teacher component contributes $g_k(h)\Delta_k(h)$. $\square$

\paragraph{Execution-time bound (Eq.~\ref{eq:mu-bound}).} On $\mathcal{E}_{\mathrm{cov}}$, $g_k(h)\Delta_k(h)\ge g_k(h)\big(\widehat\Delta_k(h)-b_k(h)\big)$ per history; on admitted histories $\widehat\Delta_k(h)>b_k(h)\ge0$ so $[\widehat\Delta_k(h)]_+=\widehat\Delta_k(h)$, and both terms vanish elsewhere. Summing yields $J(\mu_k)-J(\pi_k)\ge\widehat{\mathcal{G}}_k-\mathcal{E}_{\mathrm{verify},k}$. $\square$

\paragraph{Seat invariance (Proposition 1).} The proofs above use only the history distribution induced by the seat policy; replacing $\pi_k$ by any $\sigma$ throughout gives $J(\mu^\sigma)-J(\sigma)\ge\widehat{\mathcal{G}}^{\sigma}-\mathcal{E}_{\mathrm{verify},\sigma}$, with $\mu^\sigma=\sigma$ when $g\equiv0$ (identical process), and strict inequality when a visited history carries $g(h)>0$ (its admitted contribution is strictly positive and the verification term is controlled by admission). $\square$

\paragraph{Occupancy divergence chain.} For any $\theta$ and $t\ge1$,
\begin{equation}
\mathrm{TV}(d_{\mu_k}^{t},d_{\pi_\theta}^{t})\le\sum_{s<t}\mathbb E_{h\sim d_{\mu_k}^{s}}\big[\mathrm{TV}\big(\mu_k(\cdot\mid x,h),\,\pi_\theta(\cdot\mid x,h)\big)\big].
\label{eq:app-chain}
\end{equation}
\emph{Proof.} The two history distributions can diverge only after their first divergence point; sum over divergence times and merge same-time contributions by the recursion of the history distributions (standard simulation argument). $\square$

\paragraph{Projection-error bound.} Combining the chain with Pinsker ($\mathrm{TV}\le\sqrt{D_{\mathrm{KL}}/2}$) and Cauchy--Schwarz,
\begin{equation}
|J(\mu_k)-J(\pi_\theta)|\;\le\;C_{T,\gamma,R}\sqrt{\varepsilon_{\mathrm{proj},k}},
\label{eq:app-proj}
\end{equation}
with $\varepsilon_{\mathrm{proj},k}=\sum_t\gamma^t\mathbb E_{h\sim d_{\mu_k}^{t}}[D_{\mathrm{KL}}(\mu_k\Vert\pi_\theta)]$ and $C_{T,\gamma,R}=\bar R_{\mathrm{norm}}\sqrt{S_\gamma/2}/(1-\gamma)$ for the appropriately normalized return bound.

\paragraph{Theorem 1.} On $\mathcal{E}_{\mathrm{cov}}$, decompose $J(\pi_{k+1})-J(\pi_k)=[J(\mu_k)-J(\pi_k)]-[J(\mu_k)-J(\pi_{k+1})]$ and bound the first term by the execution-time bound and the second by the projection-error bound at $\theta$ of $\pi_{k+1}$. $\square$

\paragraph{Token-level decomposition.} When teacher and student share a tokenizer, $D_{\mathrm{KL}}(\bar\tau_k\Vert\pi_\theta)$ decomposes into per-token conditional KLs by the chain rule, with the termination symbol included in $\delta$. $\square$

\paragraph{Corollary 1 (cold-start amortization).} The cold corpus $\mathcal{D}_0$ consists of \emph{verified-completed} trajectories of $P^{\mu^{\tau}}$, so the full-trajectory empirical risk $\frac1{|\mathcal{D}_0|}\sum_\zeta-\log\pi_\theta(\zeta)$ is an unbiased surrogate of the \emph{completion-conditioned} divergence $D_{\mathrm{KL}}(P^{\mu^{\tau}}\!\mid\!C\,\Vert\,P^{\pi_\theta})$ ($C$ the completion event; trajectory-level chain decomposition), whose population optimum is the projection on the success-conditioned occupancy---not on the unconditioned $\mu^\tau$ occupancy. This conditioning is the cold-start instance of the same credit weighting: with no per-history gates yet, verification reduces to the completion indicator---the trajectory-level $w\equiv\mathbf{1}\{C\}$ of Eq.~\ref{eq:loss}---and the Theorem-1 decomposition carries the teacher-seat bound of Proposition~1 through with the projection error measured against this weighted target. The gap between the conditioned and unconditioned occupancies is exactly the selection cost that later rounds avoid by conditioning on \emph{local counterfactual} credit instead of trajectory outcomes. $\square$

\paragraph{Corollary 2 (safe abstention).} If $\underline\Delta_\sigma(h)\le\epsilon$ almost everywhere, then $g_\sigma(h)=0$, so $\mu^\sigma=\sigma$ per history (no execution-time change) and the training weight $g_\sigma(h)w_\sigma(h)\equiv0$ (no teacher term): neither execution nor training introduces unverified teacher credit. $\square$

\paragraph{Remark (rewind semantics, full).} The deployed pipeline runs two passes: the student's live attempt realizes a failure and localizes candidate boundaries; verification then rewinds to the stored snapshot $s_{\widehat h}^{\mathrm{snap}}$. We analyze the \emph{post-rewind} process: $J(\mu_k)$ measures the return of the repaired continuation from $s_{\widehat h}^{\mathrm{snap}}$ onward, and both arms of every paired comparison (teacher correction vs.\ student resample) start from the \emph{same} snapshot under shared subsequent randomness. The first pass is therefore a fixed data-collection cost common to the comparison and cancels in Eq.~\ref{eq:mu-bound}; equivalently, $\mu_k$ is a one-pass policy over histories whose state includes the snapshot stack. Formally, the augmented process has state $(s,\mathcal{Z})$ with $\mathcal{Z}$ the snapshot--evidence stack: the diagnostic evidence $z$ observed after the boundary is part of the state that \emph{selects} the intervention point, so whether to intervene at $\widehat h$ depends on information observed after $\widehat h$ in the original rollout. $J(\mu_k)$ is the expected return of the post-rewind continuation under this two-pass process; it is not the return of a causal online policy in the original environment, and the first pass cancels only because both arms of each paired comparison share it as a fixed cost. Theorem~\ref{thm:main} bounds the augmented-process return; a snapshot-free deployment that diagnoses from the current history alone---as our hardware deployment does, running the standalone student with no teacher calls---is outside this guarantee.

\paragraph{Ideal supervision allocation.} Given a budget $B$ of teacher calls and verification rollouts, the ideal allocation solves, over histories,
\begin{equation}
\max_{q_k}\;\sum_t\gamma^t\,\mathbb E_{h\sim d_{\pi_k}^{t}}\big[q_k(h)\,\Delta_k(h)\big]
\quad\text{s.t.}\quad
\sum_t\gamma^t\,\mathbb E_{h\sim d_{\pi_k}^{t}}\big[q_k(h)\,\kappa(h)\big]\le B,\;\;q_k(h)\in[0,1],
\label{eq:app-allocation}
\end{equation}
a fractional-knapsack linear program requiring $\Delta_k(h)$ and the occupancy, neither directly observable. RE-0's diagnostic localization (shrinking the candidate history set) and LCB admission (accepting only verifiable positive credit) form its practical proxy: in the empirical-LCB sense, the admitted set is exactly the histories with estimated positive credit.

\section{Experimental details and reproducibility}
\label{app:repro}

\paragraph{Training and evaluation configuration.} All students are Qwen3.8-27B with LoRA ($r{=}16$, $\alpha{=}32$, all-linear), trained with the trajectory-level hard estimator for 120 steps (batch 2, LR $10^{-5}$ constant, seed 1, bf16); refinement warm-starts from the previous generation at LR $5\times10^{-6}$. Evaluation uses 50 held-out seeds per task (temperature 0.2, top-$p$ 0.95, non-thinking mode) reset seeds from the 210000001 series and generation seeds from the 220000001 series (stride 10007), disjoint from collection. Collection verifies each contract with $K$ teacher candidates against $M$ shared-randomness paired rollouts; admission requires the two-layer LCB of Eq.~\ref{eq:lcb} to exceed $\epsilon$ under the $\alpha$-allocated simultaneous budget.

\paragraph{Verification cost.} Across the six tasks the full ledger contains 110 audited events, of which 19 were admitted. Each event verifies its selected candidate from the boundary snapshot with $M$ shared-randomness pairs (escalating $8/12/16/24/32/48$ when evidence is insufficient; observed distribution $27/50/3/8/1/21$ events, median $M{=}12$), each pair executing one student resample and one teacher continuation: the selected candidates alone consumed $4{,}192$ suffix rollouts in total. Rejected candidates and the live first pass add further environment rollouts, so verified rows are information-dense but not rollout-free.

\section{Main-result accounting and robustness}
\label{app:robust}

\textbf{Table~\ref{tab:main} notes.} (a) Base row uses the standard 50-seed evaluation cohort, disjoint from collection. (b) The refine cell for restack comes from an intermediate teacher-seat recollection; a collection-seat artifact affecting this cell is disclosed in the Limitations. (c) Spill saturates at stage one; its counterfactual admission set is empty, so the criterion abstains---the runtime instance of Corollary~2. (d) Lift refinement consumed a single admitted event ($w{=}0.069$) plus replay: 76\%$\to$58\%, with collateral collapse on the branch (spill $28\to0$)---headroom without credit mass is harmful (Remark~\ref{rem:twoc}). (e) Cubestack is saturated at base; the system exempts it from training. The transfer table's own-task cells for $\pi_1/\pi_2'$ (restack) are 31/40 (Table~\ref{tab:main}). (f) Composition arms (restack, 50 seeds): plain SFT on 8 success-filtered teacher-seat demonstrations scores 3/50; SFT on outcome-selected patches (3 rows, unweighted) scores 46/50. The nine patches rejected by the outcome rule all have intervention success 0.0---zero divergence from the LCB gate on this corpus; both arms are single runs. (g) Row counts and admission counts are different objects: cold rows count verified-completed \emph{teacher-seat} trajectories (the Corollary~1 corpus), while Figure~\ref{fig:admission}b counts admission events in \emph{student-seat} collection rounds.

\paragraph{Per-round sign accounting.}
Theorem~\ref{thm:main} bounds the standalone gain by the verified intervention gain net of verification and projection terms; the first two are measured at admission, while the projection term is guarded operationally by the outer acceptance criterion (Remark~\ref{rem:scope}). Table~\ref{tab:projgap} reports, for every round in which the criterion made a retain-or-halt decision, the admitted credit of the round against the realized change in standalone performance on the 50-seed protocol. The signs agree in all four rounds, including the round the criterion rejected. We read this as sign-level consistency evidence over four rounds, not an estimate of $\varepsilon_{\mathrm{proj},k}$. Re-executing the rejected lift round with training seed 2 reproduces the regression (76\%$\to$46\%), so the halt reflects the round's credit rather than a seed artifact.

\begin{table}[h]
\centering\small
\begin{tabular}{lllll}
\toprule
Round & Admitted credit & $\widehat{J}$ change & Decision & Retained \\
\midrule
restack $\pi_0\!\to\!\pi_1$ (cold) & 6/6 events, $\Dhat\!\approx\!+1.0$ & $+54$\,pt ($8\!\to\!62\%$) & $J<\rho$: continue & $\pi_1$ \\
restack $\pi_1\!\to\!\pi_2'$ (refine) & 4/15, $\Dhat\!\in\![0.19,0.88]$ & $+18$\,pt ($62\!\to\!80\%$) & $J\ge\rho$: halt & $\pi_2'$ \\
lift $\pi_1\!\to\!\pi_2$ (refine) & 1 event, $w{=}0.069$ & $-18$\,pt ($76\!\to\!58\%$) & regression: halt & $\pi_1$ \\
spill $\pi_1\!\to$ (refine) & 0 admitted & --- (no projection) & abstain (Cor.~2) & $\pi_1$ \\
\bottomrule
\end{tabular}
\caption{Per-round sign accounting: admitted credit against realized standalone change. The retained generation is the measured-better one in every round.}
\label{tab:projgap}
\end{table}

\paragraph{Wilson confidence intervals.}
\begin{table}[h]
\centering\footnotesize
\begin{tabular}{lccc}
\toprule
Task & Base & RE-OPD-cold & RE-OPD-refine \\
\midrule
restack & [3.2, 18.8] & [48.2, 74.1] & [67.0, 88.8] \\
spill & [54.2, 79.2] & [92.9, 100.0] & --- \\
nut & [1.1, 13.5] & [76.2, 94.4] & --- \\
lift & [1.1, 13.5] & [62.6, 85.7] & [44.2, 70.6] \\
door & [86.5, 98.9] & [92.9, 100.0] & --- \\
cubestack & [92.9, 100.0] & [92.9, 100.0] & --- \\
\bottomrule
\end{tabular}
\caption{Wilson 95\% CIs (\%) for Table~\ref{tab:main}.}
\end{table}

\paragraph{Trainer-seed replication.}
\label{app:trseed}
We retrained the restack and spill cold-start arms with two additional trainer seeds on the identical corpora and recipe. Restack: seed 1 $=31/50$, seed 2 $=44/50$, seed 3 $=36/50$ (mean 37.0)---every replicate far above the base 4/50. Spill: seed 1 $=50/50$, seeds 2 and 3 $=33/50$ each (mean 38.7)---two of three replicates sit at base level (34/50), so the spill cold-start gain is seed-sensitive. With the same protocol we additionally replicated the remaining cold arms: lift reaches $38/33/35$ (mean 35.3), nut $44/40/43$ (mean 42.3), and door $50/50/50$---every replicate sits far above its base, so the large cold-start gains on the thin-corpus tasks are robust to trainer seed, and spill remains the sole seed-sensitive task.

\paragraph{Never-used cohort re-evaluation.}
Base and cold checkpoints were additionally re-evaluated on a disjoint, never-used 50-seed cohort (reset offset 230000001, generation offset 240000001): base $3/3/5/47/50$ (restack, nut, lift, door, cubestack) and cold $28/45/42/50/49$ (restack, nut, lift, door, spill) match the reported cells within binomial noise, confirming that the reported numbers are not artifacts of the evaluation-seed protocol.

\section{Mechanism and ablation details}
\label{app:selection}

Table~\ref{tab:selection} holds each task's collection journal fixed and varies only the admission rule. Provenance: the lift/spill/nut journals are teacher-seat (self-play) collections; all arms use identical row construction (uniform weight, full-trajectory SFT). The four lift events admitted by the sign rule ($\Dhat>0$) but rejected by the LCB have $\Dhat\in\{+0.21,+0.06,+0.04,+0.00004\}$ at $M\in\{48,48,48,12\}$: positive in mean, statistically indistinguishable from zero. Table~\ref{tab:selectionfull} completes the grid with direction cells and the restack self-imitation arm.

\begin{table}[h]
\centering\small
\begin{tabular}{llccl}
\toprule
Task / rule & Rows & Seed 1 & Seed 2 & Note \\
\midrule
lift: none & 25 & 6 & 2 & \\
lift: sign & 7 & 5 & 4 & \\
lift: LCB & 3 & \textbf{11} & \textbf{5} & \\
lift: outcome & 27 & 5 & 9 & \\
spill: none & 6 & 10 & --- & mean reward 0.63 \\
spill: sign & 2 & 0 & --- & collapse; reward 0.00 \\
spill, nut: LCB & 0 & --- & --- & abstain \\
nut: none & 4 & 0 & --- & \\
nut: sign & 3 & 2 & --- & \\
restack: self-imitation & 11 & 2 & --- & RFT analogue \\
\bottomrule
\end{tabular}
\caption{Full selection-ablation grid (successes/50; one run per seed).}
\label{tab:selectionfull}
\end{table}

\noindent The spill sign arm collapses to 0/50 with exactly zero reward: all fifty rollouts execute the same three-block program with no wipe coverage---two near-duplicate rows reduce the policy to template repetition. The LCB admits nothing on spill or nut and abstains. True base-seat RFT is infeasible in our protocol; the nearest analogue---self-imitation on the 11 success programs of the $\pi_1$-seat policy---reaches 2/50 under identical row construction.

\paragraph{Teacher-effect ladder.}
\label{app:ladder}
Table~\ref{tab:ladder} details the admission-ledger grading: student self-play yields no verifiable credit; narrow local task experts (win rates 62--100\% on their own tasks) fail deterministically at the contract-format layer with a $\sim$10\% rescue ceiling; the strong API teacher earns 4 admissions on 13 events with rescue rates 19--47\%.

\begin{table}[h]
\caption{\textbf{The admission ledger as a teacher-effect instrument.}}
\label{tab:ladder}
\centering\footnotesize
\begin{tabular}{lcccl}
\toprule
Teacher & Events & Admitted & Rescue rate & Failure mode \\
\midrule
Student self-play & 33$^{+}$ & \textbf{0} & $\sim$0 & no credit at $p\!\approx\!0$ \\
Local task experts & 32 & \textbf{0} & $\sim$10\% ceiling & contract-format failures \\
API frontier & 13 & \textbf{4} & 19--47\% & --- \\
\bottomrule
\end{tabular}
{\footnotesize $^{+}$restack 11, nut 6, spill 5.}
\end{table}

\section{Transfer, pooling, and deployment}
\label{app:transfer}

\begin{table}[h]
\caption{\textbf{Cross-task transfer} (successes / 50, @120).}
\label{tab:transfer}
\centering\footnotesize
\begin{tabular}{lcccc}
\toprule
Policy $\backslash$ Task & nut & spill & lift & handover \\
\midrule
Base & 2 & 34 & 2 & --- \\
$\pi_1$ (restack student) & 22 & 24 & 6 & 0 \\
$\pi_2'$ (restack refined) & 16 & \textbf{45} & 3 & 1 \\
Lift expert & 0 & 28 & \textbf{38} & 0 \\
\bottomrule
\end{tabular}
{\footnotesize Own-task (restack): 31/40.}
\end{table}

\begin{table}[h]
\caption{\textbf{Pooled joint training fails} (successes / 50). Two separable failures: task interference and quota starvation.}
\label{tab:pooling}
\centering\footnotesize
\begin{tabular}{lccccc}
\toprule
Task & Expert & Equal (42) & Proportional (25) & Full-pool (172) & Credit-equal \\
\midrule
restack & \textbf{31} & 3 & 2 & 2 & 1 \\
spill & \textbf{50} & 31 & 0 & 0 & 32 \\
nut & \textbf{44} & 9 & 0 & 0 & 21 \\
lift & \textbf{38} & 20 & 25 & 5 & 15 \\
cubestack & 50 & 50 & 0 & 48 & 48 \\
door & 50 & 50 & 0 & 49 & 49 \\
\bottomrule
\end{tabular}
\end{table}

\paragraph{Real-robot platform.}
\label{app:realrobot}
The deployment platform is a dual-arm AgileX PiPER rig with SAM3-based open-vocabulary perception (Figure~\ref{fig:piper}); the execution stack exposes the simulation API surface unchanged, so the deployed policy code requires no hardware-specific edits.

\paragraph{Compute infrastructure.}
All simulation training, collection, and evaluation ran on an eight-GPU NVIDIA H200 server (143\,GB per GPU), one GPU per run: the 27B student is served with vLLM (LoRA adapters, runtime hot-loading), training uses FSDP with gradient checkpointing, and verification rollouts run against the same engine; teacher and diagnosis calls go to an external API model.

\begin{figure}[h]
\centering
\includegraphics[width=0.62\linewidth]{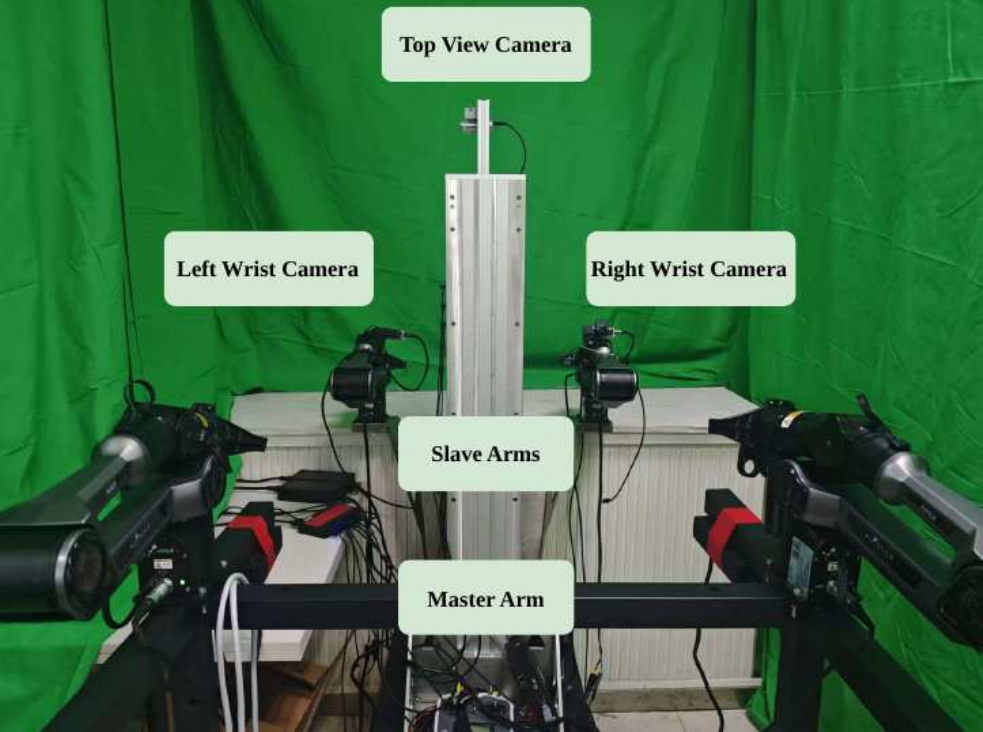}
\caption{\textbf{Real-robot platform.} The dual-arm AgileX PiPER rig used for the zero-shot deployment experiments.}
\label{fig:piper}
\end{figure}

\section{Baseline details}
\label{app:grpo}

\paragraph{GRPO configuration and fairness.} The GRPO control was tuned until it trained at all in this stack: chat-template thinking mode disabled (otherwise every response clips at the token budget and all rewards are zero), max response length 3072, thread-safety fixes for the robosuite reward and IK-server liveness monitoring, and memory-feasible two-card colocate settings for the 27B policy. KL regularization (reference model) was disabled to fit GPU memory; entropy collapse (0.30$\to$0.002) persisted across all of these fixes. We cannot rule out that a larger-memory configuration with KL regularization behaves differently; the comparison should be read as against a resource-matched GRPO, not a maximally tuned one.
The GRPO control evaluates on the same 50-seed protocol: restack 1/50, spill 0/50, lift 0/50, cubestack 50/50 (base 4/34/2/50); the nut evaluation could not be completed due to a serving-stack incompatibility and is omitted.

\end{document}